\documentclass{article}

\usepackage[preprint]{neurips_2026}

\usepackage{amsmath}
\usepackage{amssymb}        
\usepackage{amsfonts}       
\usepackage{mathtools}
\usepackage{amsthm}
\usepackage{enumitem}
\usepackage[utf8]{inputenc}
\usepackage[T1]{fontenc}
\usepackage{microtype}
\usepackage{nicefrac}

\usepackage{booktabs}
\usepackage{multirow}
\usepackage{array}
\usepackage{makecell}
\usepackage{longtable}
\usepackage{tabularx}
\usepackage{colortbl}
\usepackage[tableposition=top]{caption}

\usepackage{graphicx}
\usepackage{subcaption}
\usepackage{wrapfig}
\usepackage{float}
\usepackage[figuresleft]{rotating}
\usepackage[export]{adjustbox}

\usepackage{xcolor}
\definecolor{Gray}{gray}{0.92}
\definecolor{LightBlue}{rgb}{0.95, 0.96, 0.98}
\definecolor{highlightblue}{RGB}{235, 245, 255}
\definecolor{promptbgcolor}{RGB}{245,245,245}
\definecolor{promptframecolor}{RGB}{180,180,180}

\usepackage{algpseudocode}
\usepackage{algorithm}
\usepackage{listings}
\usepackage[cachedir=minted-cache]{minted}

\usepackage{hyperref}
\usepackage{url}
\usepackage[capitalize,noabbrev]{cleveref}

\usepackage{pifont}         
\usepackage[most,breakable]{tcolorbox}
\usepackage[textsize=tiny]{todonotes}

\newtcolorbox{promptbox}[2][]{
  colback=promptbgcolor,
  colframe=promptframecolor,
  fonttitle=\bfseries\small,
  title={#2},
  enhanced,
  breakable,
  fontupper=\small,
  attach boxed title to top left={yshift=-2mm, xshift=4mm},
  boxed title style={colback=promptframecolor, colframe=promptframecolor, coltitle=white},
  before upper={\setlength{\parskip}{3pt}},
  #1
}

\newcommand{\templatevar}[1]{\textcolor{blue}{\texttt{\{\{#1\}\}}}}

\theoremstyle{plain}
\newtheorem{theorem}{Theorem}[section]

\newtheorem{lemma}[theorem]{Lemma}

\theoremstyle{definition}

\theoremstyle{remark}

\title{From Static Context to Calibrated Interactive RL: Mitigating Distribution Shift in Multi-turn Dialogue with Aligned Simulator}

\author{
\centering
\begin{minipage}{0.95\textwidth}
\centering
Xiaohua Wang\textsuperscript{\thanks{These authors contributed equally.}} \quad
Jiakang Yuan\textsuperscript{\footnotemark[1]} \quad
Zisu Huang \quad
Muzhao Tian \quad
Changze Lv \\
Kaitao Song \quad
Chen Tao \quad
Xiaoqing Zheng\textsuperscript{\thanks{Corresponding author.}} \quad \\
\vspace{0.5em}
\normalfont{
\textsuperscript{}Fudan University \quad
\textsuperscript{}} \\
\vspace{0.5em}
\texttt{xhwang24@m.fudan.edu.cn}
\end{minipage}
}

\begin{document}

\maketitle
\begin{abstract}

A long-standing goal of the research community is to develop highly interactive LLM-based dialogue agents. Recent research focuses on optimizing policies based on fixed offline logs (Static Context RL) or using a prompt-based simulator (Interactive RL). In this work, we theoretically show that both paradigms are fundamentally limited by context distribution shift—a mismatch between dialogue histories observed during training and those encountered in real conversations. This shift compounds quadratically over turns and severely degrades dialogue quality. Specifically, we attribute this shift to two distinct sources: (i) policy-induced shift, arising from training on static histories rather than self-generated trajectories; and (ii) simulator-induced shift, stemming from discrepancies between simulated and real human behaviors. To address these challenges, we propose Calibrated Interactive RL, a unified framework that couples interactive RL with simulator alignment. By aligning the simulator with human interaction patterns, our approach reduces the sim-to-real gap and mitigates compounding distribution shifts. Experiments across multiple dialogue tasks confirm our theoretical analysis: (i) Interactive RL significantly outperforms the Static Context baseline by mitigating policy distribution shift; and (ii) calibrating simulators with our alignment method further bridges the sim-to-real gap, yielding state-of-the-art downstream performance.

\end{abstract}

\section{Introduction}
\label{sec:introduction}
A long-standing goal in the research community is to develop highly interactive LLM-based dialogue agents that can sustain coherent, safe, and goal-directed multi-turn conversations~\cite{li2025beyond,laban2025llms}. To advance this objective, recent work has increasingly adopted reinforcement learning (RL) to align the model's behavior with human preferences~\cite{ouyang2022training,bai2022training,bai2022constitutional}. Broadly, existing approaches can be categorized into two primary paradigms: \textbf{\textit{Static Context RL}}~\cite{lee2023rlaif,hong2023zero,verma2022chai,snell2022offline,rafailov2023direct} and \textbf{\textit{Interactive RL}}~\cite{wu2025collabllm,shani2024multi,carta2023grounding}. The former optimizes policy models using fixed offline contexts (\textit{i.e.}, serialized conversation histories), without the need to generate the preceding turns. In contrast, the latter addresses the inherently dynamic nature of dialogue by employing a prompt-based user simulator, thereby enabling agents to engage in multi-turn interactions within a closed-loop learning framework.

Despite their preliminary efficacy, these approaches face fundamental limitations as shown in Table~\ref{tab:paradigm_comparison}. On one hand, static context methods disrupt the essential RL interaction loop, thereby undermining the interactive capabilities of dialogue agents. On the other hand, while interactive approaches attempt to bridge this gap, they typically rely on prompt-based simulators that fail to capture the stochasticity, ambiguity, and noise inherent in real user inputs. This deficiency creates a significant sim-to-real gap, ultimately compromising the policy's applicability in real-world scenarios.

In this work, we establish the theoretical foundations for addressing these challenges. We formally demonstrate that existing paradigms are fundamentally constrained by context distribution shift—a discrepancy between dialogue histories observed during training and those encountered in real conversations. We show that this shift compounds quadratically over turns, severely degrading dialogue quality. Specifically, we decompose this phenomenon into two distinct sources: (i) \textit{\textbf{Policy-induced shift}} is caused by the mismatch between fixed offline histories and self-generated trajectories; and (ii) \textit{\textbf{Simulator-induced shift}} arises from the structural divergence between prompt-based simulators and real human behaviors. 

To this end, we introduce \textbf{\textit{Calibrated Interactive RL}}, a unified framework that couples interactive RL with the aligned simulator. Specifically, we first align the user simulator using Supervised Fine-Tuning (SFT) on real human multi-turn interactions, enabling it to capture realistic human behaviors such as clarification-seeking, stubbornness, and strict goal-adherence. Building upon this calibrated simulator, we leverage interactive RL to optimize the policy on self-generated trajectories, thereby explicitly addressing the policy-induced distribution shift and allowing the agent to learn essential error-recovery skills.

Empirically, our analysis reveals that standard prompt-based simulators exhibit severe behavioral artifacts, most notably \textit{sycophancy}—a tendency to blindly accept incorrect agent outputs or prematurely leak ground-truth information. This effectively trivializes the RL interaction loop and inevitably leads to reward hacking, where the policy learns to exploit the simulator's permissiveness rather than solving the actual task. By treating the simulator as a trainable agent and performing behavioral cloning on high-quality offline logs, our Simulator Alignment framework rigorously grounds the simulation in human reality.
Our experiments demonstrate that Interactive RL consistently outperforms the Static Context baseline and that training with our Aligned Simulators yields further significant gains. These results highlight the need for both closing the interaction loop and explicitly grounding simulations. 

Our contributions are summarized as follows:

\begin{itemize}[leftmargin=*]
\item We establish the theoretical foundation for multi-turn dialogue RL, reveal the source of the context distribution shift in existing RL methods, and mathematically demonstrate that these shifts compound quadratically over turns.
\item We propose \textbf{Calibrated Interactive RL}, a unified framework that couples interactive policy optimization with simulator alignment. We introduce an SFT-based calibration method that grounds the user simulator in real human behavioral distributions, eliminating simulation artifacts (\textit{e.g.}, sycophancy) and effectively bridging the sim-to-real gap.
\item Extensive experiments demonstrate that Interactive RL consistently outperforms Static Context baselines. Furthermore, explicitly calibrating the simulator yields significant additional gains, highlighting the critical importance of preventing reward hacking by grounding simulations in human reality.
\end{itemize}

\section{Related Work}


\begin{table*}[t]
    \centering
    \renewcommand{\arraystretch}{1.4} 
    \caption{\textbf{Comparison of RL Paradigms in Multi-turn Dialogue.} We contrast our proposed Calibrated Interactive RL with existing Static and Interactive baselines. Our framework is the only one that addresses both policy-induced ($\Delta_t$) and simulator-induced ($\delta_t$) distribution shifts.}
    \label{tab:paradigm_comparison}
    \vspace{2mm}
    
    \resizebox{\linewidth}{!}{%
    \begin{tabular}{l l c c m{7.5cm}}
        \toprule[1.5pt]
        \textbf{Paradigm} & \textbf{Training Context} & \textbf{Interaction Loop} & \textbf{Simulator Fidelity} & \textbf{Primary Limitation (Theoretical)} \\
        \midrule
        
        \textbf{Static Context RL} & Offline Logs (Fixed) & Open Loop & N/A 
        & \textbf{Policy Shift ($\Delta_t$):} Errors compound quadratically; no error recovery training. \\
        
        \textbf{Interactive RL} & On-Policy Trajectories & Closed Loop & Uncalibrated 
        & \textbf{Sim-to-Real Gap ($\delta_t$):} Simulator artifacts (e.g., sycophancy) mislead the policy. \\
        
        \rowcolor{LightBlue} 
        \textbf{Calibrated RL (Ours)} & On-Policy Trajectories & Closed Loop & \textbf{Aligned} 
        & \textbf{None:} Mitigates both shifts via joint interactive training and simulator alignment. \\
        \bottomrule[1.5pt]
    \end{tabular}
    }
\end{table*}

\subsection{Reinforcement Learning for Multi-turn Dialogue}

Reinforcement Learning from Human Feedback (RLHF) has established itself as the cornerstone for aligning LLMs with human intent~\cite{ouyang2022training, bai2022training}. In the context of dialogue systems, existing approaches can be broadly categorized into \textit{Static Context RL} and \textit{Interactive RL}. \textbf{\textit{Static Context RL}} optimizes policies based on static, offline conversation logs~\cite{shi2024direct,hong2023zero,snell2022offline}. For instance, Charlie et al.~\citep{snell2022offline} applies implicit Q-learning to fixed language datasets. Although effective for single-turn instruction following, these static context approaches suffer from distribution shift and the inability to correct generation errors that compound over turns—a known limitation in sequential decision-making~\cite{levine2020offline}. In contrast, \textbf{\textit{Interactive RL}} mainly utilizes a simulator to close the interaction loop and allows the agent to learn from its own generation trajectory~\cite{shani2024multi,zhou2024archer}. For example, CollabLLM~\cite{wu2025collabllm} utilizes a prompt-based user simulator to train agents on online collaborative traces. Distinct from existing approaches, we theoretically prove that existing methods suffer from context distribution shift between training and those encountered in real conversations. Besides, \textbf{\textit{Calibrated Interactive RL}} is introduced to bridge the sim-to-real gap by aligning the user simulator via discriminator-derived realism rewards.

\subsection{Distribution Shift and Exposure Bias}
The discrepancy between training on ground-truth history (\textit{i.e.}, teacher forcing) and generating on the model's own history (\textit{i.e.}, student forcing) is a fundamental challenge in sequence generation, known as exposure bias~\cite{bengio2015scheduled, ranzato2015sequence}. 
In the context of decision-making, this phenomenon parallels the distribution shift analyzed in imitation learning. Ross et al.~\cite{ross2011reduction} demonstrated that behavior cloning (\textit{i.e.}, offline training) results in errors that accumulate quadratically ($O(H^2)$) with the horizon $H$, whereas interactive expert guidance (\textit{i.e.}, \textsc{Dagger}) can reduce this to linear growth. 
While recent works have analyzed this shift in the context of single-turn reasoning chains~\cite{wang2023robustness}, our work formally extends this analysis to the multi-turn dialogue setting. We derive a rigorous bound illustrating how static context training fails to generalize to the dynamic distribution induced by the policy's own actions during prolonged interactions.

\subsection{LLM-based User Simulators}
Employing LLM-based simulators as surrogates for human users has emerged as a prominent trend for scalable evaluation and training~\cite{shi2019build,tseng2021transferable,hu2023unlocking,gandhi2026learning}. 
Park et al. \cite{park2023generative} demonstrated that LLM agents can simulate believable social behaviors, and automated evaluation frameworks like MT-Bench~\cite{bai2024mt} and AlpacaEval~\cite{dubois2023alpacafarm} utilize LLM judges to proxy human preferences. 
In the domain of interactive training, recent studies have employed "Self-Play" or simulator-based RL to improve task performance~\cite{shao2023character,chen2024self,silver2017mastering}. 
However, uncalibrated simulators often exhibit distinct artifacts, such as "sycophancy" (agreeing with the model's errors)~\cite{sharma2023towards} or unrealistic persona consistency~\cite{shanahan2023role}. These artifacts can be exploited by the policy agent during RL training, leading to reward hacking~\cite{wang2025reward} where the agent optimizes for simulator approval rather than genuine task success.
Unlike previous works that use off-the-shelf LLMs as simulators, we propose a \textit{Simulator Alignment} framework, treating the simulator as a trainable agent optimized to minimize the sim-to-real gap, thereby ensuring that interactive training yields robust real-world performance.

\section{Distribution Shift in Multi-turn RL}
\label{sec:theory_main}
\begin{figure*}[t]
  \begin{center}
    \centerline{\includegraphics[width=1.0\linewidth]{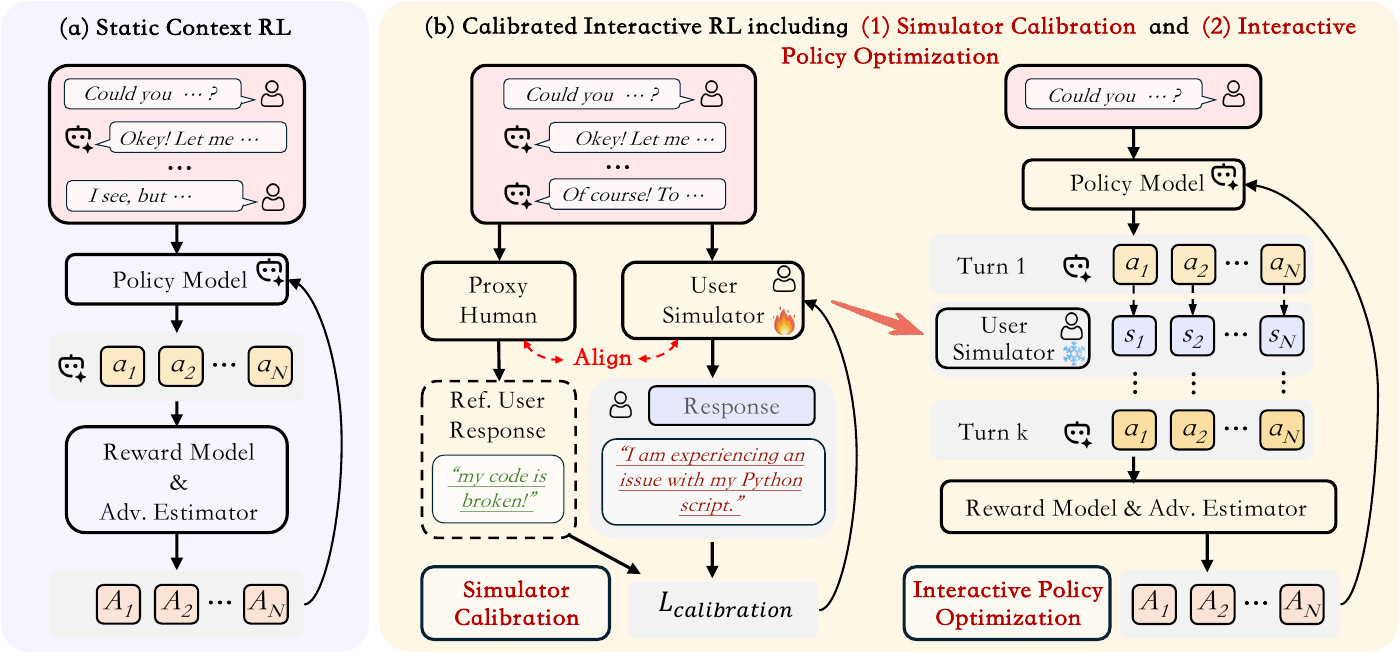}}
    \caption{\textbf{Comparison of reinforcement learning frameworks for multi-turn dialogue.} (a) Pseudo Multi-turn RL optimizes the policy using static offline conversation logs, where the model conditions its actions on fixed historical contexts without generating the preceding turns, leading to distribution shift. (b) True Multi-turn RL (Ours) closes the interaction loop through a two-stage process: (1) User Simulator Training, where the simulator is explicitly aligned with human reference data to minimize the simulator–human gap; and (2) Policy Training, where the agent is optimized via dynamic, multi-turn interactions with the aligned simulator, ensuring robustness to compounding errors.}
    \label{fig:fig_1}
  \end{center}
  \vspace{-5mm}
\end{figure*}

In this section, we formally analyze the distribution shifts inherent in multi-turn dialogue training. We derive two fundamental bounds that motivate our proposed \textbf{Calibrated Interactive RL} framework. First, we show that standard offline training (\textit{i.e.}, \textit{Static Context RL}) induces a biased objective due to compounding policy deviations. Second, we demonstrate that while \textit{Interactive RL} addresses this by closing the interaction loop, it relies on simulator fidelity, as simulator errors also accumulate quadratically. These insights underscore the necessity of our Simulator Alignment approach. Detailed proofs are provided in Appendix~\ref{app:theory-proofs}.

\subsection{Setup and Notation}
We consider an episodic dialogue environment with a finite horizon $H$. Let $s_t \in \mathcal{S}$ denote the environment state and $a_t \in \mathcal{A}$ the dialogue agent's action at time $t$. A history $h_t = (s_0, a_0, \dots, s_t)$ represents the trajectory up to time $t$. Policies map histories to distributions over actions; we denote the behavior policy that generated the offline logs as $\pi_{\text{data}}$, and the learned policy as $\pi$.

Let $d^\pi_t(h)$ be the marginal distribution over histories at time $t$ induced by policy $\pi$. To analyze the simulation gap, we distinguish between the transition dynamics of the real user, denoted by the operator $\mathcal{T}_{\text{real}}$, and the user simulator, $\mathcal{T}_{\text{sim}}$. We assume rewards are bounded such that $|r(s, a)| \le R_{\max}$. The objective is to maximize the finite-horizon return $J(\pi) = \mathbb{E}_{\tau \sim p_\pi}\left[\sum_{t=0}^{H-1} r_t\right]$.

Based on the notation, \textbf{Static Context RL} decouples policy optimization from state generation. Its updates and gradient estimators are computed using contexts $h \sim d^{\text{data}}_t$ sampled from the offline distribution $d^{\text{data}}_t$, rather than the on-policy distribution $d^\pi_t$. Consequently, the optimization is fundamentally off-policy with respect to the state distribution, as it ignores how $\pi$ influences future histories.
On the contrary, \textbf{Interactive RL} maintains the closed interaction loop. The objective is evaluated under the distribution $h \sim d^\pi_t$ induced by the policy's interaction with the environment, thereby capturing the consequences of its own actions.

\subsection{Error Source I: Compounding Policy Shift in Static Context RL}
We first quantify the performance degradation caused by training on static contexts. Let $\pi'$ denote a reference policy (\textit{e.g.}, the data-generating policy $\pi_{\text{data}}$) and $\pi$ denote the learned policy. We define two divergence terms: the \textit{history distribution mismatch} $\Delta_t = \|d^\pi_t - d^{\pi'}_t\|_1$, and the \textit{per-step policy deviation} $\varepsilon_t = \mathbb{E}_{h \sim d^{\pi'}_t} \|\pi(\cdot \mid h) - \pi'(\cdot \mid h)\|_1$.

\begin{theorem}[Performance Difference via Policy Shift]
\label{thm:policy_shift}
For any two policies $\pi, \pi'$ and finite horizon $H$, the difference in expected return is bounded by:
\begin{equation}
\small
|J(\pi) - J(\pi')| \le R_{\max} \sum_{t=0}^{H-1} (\Delta_t + \varepsilon_t).
\end{equation}
Furthermore, since $\Delta_t \le \sum_{i=0}^{t-1} \varepsilon_i$, the total error satisfies:
\begin{equation}
\small
|J(\pi) - J(\pi')| \le R_{\max} \sum_{i=0}^{H-1} (H - i) \varepsilon_i.
\end{equation}
\end{theorem}

\textbf{Proof Sketch.} We decompose the return difference into a sum over time of divergences in the joint (history, action) distributions. As the environment transition operator is non-expansive in the $\ell_1$-norm, distribution errors propagate linearly, yielding $\Delta_{t+1} \le \Delta_t + \varepsilon_t$. Summing this recurrence yields the quadratic bound. Detailed proofs are provided in \cref{app:theory-proofs}.

\textbf{Theoretical Implication.} Theorem~\ref{thm:policy_shift} establishes that errors in Static Context RL compound quadratically ($O(H^2)$). Small deviations $\varepsilon_t$ in the policy at early turns accumulate into a large distribution shift $\Delta_t$, pushing the agent into out-of-distribution states where it has never been trained. Static Context RL ignores $\Delta_t$, optimizing only for local accuracy, and thus fails to guarantee long-horizon interactive performance.

\textbf{Intuitive Explanation \& Running Example.} To intuitively understand this quadratic compounding effect, consider a conversational math tutoring scenario. Suppose the agent makes a minor logical error at turn $t=1$ (a small policy deviation $\varepsilon_1$). In a real interaction, the subsequent dialogue history will fundamentally diverge from the ground-truth offline logs (creating a large history mismatch $\Delta_2$), because the human user will likely point out the mistake or express confusion. However, a model trained purely via Static Context RL has only seen perfect, expert-guided trajectory histories. When thrust into this self-generated, flawed context during deployment, the agent lacks the essential \textit{error-recovery} skills. Instead of correcting itself, it is prone to hallucinate further (yielding large $\varepsilon_2, \varepsilon_3, \dots$). Consequently, an initially negligible mistake cascades dynamically over the horizon $H$, leading to catastrophic failure. Interactive RL mitigates this by explicitly exposing the agent to its own mistakes during training, thereby teaching it how to navigate back to the correct path.

\subsection{Error Source II: Simulator-Real Distribution Shift}
While Interactive RL resolves the policy distribution shift by training on $d^\pi_t$, it introduces a new source of error: the discrepancy between the simulator $\mathcal{T}_{\text{sim}}$ and the real user $\mathcal{T}_{\text{real}}$. To isolate the impact of simulator fidelity, we analyze the gap for a fixed policy $\pi$. We define the \textit{simulator error} at time $t$ as the expected divergence in transition dynamics under the real joint distribution: $\delta_t = \mathbb{E}_{(h,a) \sim q^{\text{real}}_t} \|\mathcal{T}_{\text{real}}(\cdot|h, a) - \mathcal{T}_{\text{sim}}(\cdot|h, a)\|_1$, where $q^{\text{real}}_t(h,a) = d^{\pi}_{\text{real},t}(h)\pi(a|h)$.

\begin{theorem}[Sim-to-Real Gap]
\label{thm:sim_gap}
For any policy $\pi$ and horizon $H$, the performance gap caused by the simulator is bounded by:
\vspace{-3mm}
\begin{equation}
\small
|J_{\text{real}}(\pi) - J_{\text{sim}}(\pi)| \le R_{\max} \sum_{i=0}^{H-1} (H - i) \delta_i.
\end{equation}
\vspace{-5mm}
\end{theorem}

\textbf{Proof Sketch} We upper-bound the performance gap between the real and simulated environments by aggregating the marginal distribution shifts over the horizon. Leveraging the non-expansive nature of stochastic transition operators under the $\ell_1$-norm, we show that the distribution divergence at any step is bounded by the sum of the propagated error from the previous step and the immediate local simulation error. This leads to the recurrence $\Delta_{t+1}^{\text{env}} \leq \Delta_t^{\text{env}} + \delta_t$, which telescopes to yield the final cumulative bound. The complete derivation is detailed in \cref{app:theory-proofs}.

\begin{wrapfigure}[22]{R}{0.54\textwidth} 
\begin{minipage}{\linewidth}
\vspace{-10mm}
\begin{algorithm}[H] 
   \caption{Calibrated Interactive RL}
   \label{alg:calibrated_rl}
\begin{algorithmic}[1]
   \State {\bfseries Input:} Multi-turn data $\mathcal{D}_{\text{multi}}$, Task data $\mathcal{D}$
   \State {\bfseries Initialize:} Simulator $\pi^{\text{sim}}_\phi$, Policy $\pi_\theta$
   
   \State \textcolor{gray}{\textit{// Phase 1: Simulator Calibration (SFT)}}
   \While{not converged}
       \State Sample batch $(x, y^*, \tau_{\text{human}})$ from $\mathcal{D}_{\text{multi}}$
       \State Compute $\mathcal{L}_{\text{SFT}}(\phi)$ via NLL on user turns $u_t$
       \State Update $\phi$ via Gradient Descent
   \EndWhile
   \State Freeze aligned simulator $\pi^{\text{sim}}_{\phi^*}$
   
   \State \textcolor{gray}{\textit{// Phase 2: Interactive Policy Opt.}}
   \For{each training epoch}
       \State Sample context $(x, y^*)$ from $\mathcal{D}$
       \State Initialize history $h_0 = x$
       \For{$t = 0$ to $H-1$}
           \State Agent: $a_t \sim \pi_\theta(\cdot \mid h_t)$
           \State Simulator: $u_t \sim \pi^{\text{sim}}_{\phi^*}(\cdot \mid[x, y^*, h_t, a_t])$
           \State Update history $h_{t+1} \leftarrow h_t \cup \{a_t, u_t\}$
       \EndFor
       \State Final reward $R = \text{Score}(a_{H-1}, y^*)$
       \State Update policy $\theta$ via GRPO to maximize $R$
   \EndFor
\end{algorithmic}
\end{algorithm}
\end{minipage}
\end{wrapfigure}

\textbf{Theoretical Implication.} This result highlights that simulator misalignment is as detrimental as policy misalignment. Even if interactive training is used, a simulator that deviates from human behavior (i.e., $\delta_t > 0$) causes the training objective to diverge from the real-world objective at a rate of $O(H^2 \delta)$. This formally justifies our \textbf{Simulator Alignment} approach: minimizing $\delta_t$ is a theoretical prerequisite for robust interactive learning.

\textbf{Intuitive Explanation \& Running Example.} To see why standard interactive training is insufficient without simulator calibration, consider the well-documented behavioral artifacts of uncalibrated LLM simulators, such as ``sycophancy'' (blind agreement)~\cite{sharma2023towards}. Suppose the dialogue agent proposes a highly flawed document edit. A real human user would reject the edit and demand a revision (representing a specific transition in $\mathcal{T}_{\text{real}}$). In contrast, a naive, uncalibrated simulator might enthusiastically accept the flawed edit, complimenting the agent (a severely divergent transition in $\mathcal{T}_{\text{sim}}$, yielding a large local error $\delta_t$). By interacting with this flawed simulator over multiple turns, the agent receives false-positive feedback and learns a degenerate policy that exploits the simulator's permissiveness rather than actually solving the task---a phenomenon closely related to \textit{reward hacking}~\cite{wang2025reward}. As shown by our bound, this simulator-induced delusion accumulates quadratically over the conversation length $H$. Therefore, aligning the simulator to accurately reflect human stubbornness, constraints, and feedback patterns is not merely an engineering trick, but a fundamental mathematical requirement to ensure that the surrogate objective $J_{\text{sim}}(\pi)$ faithfully proxies the true real-world objective $J_{\text{real}}(\pi)$.




\section{Calibrated Interactive RL}
\label{sec:method}

Motivated by the compounding distribution shifts identified in \cref{thm:policy_shift} and \cref{thm:sim_gap}, it is evident that a robust multi-turn dialogue agent must simultaneously address both policy-induced and simulator-induced errors. To this end, we propose \textbf{Calibrated Interactive RL}. As shown in \cref{fig:fig_1} (b), our method operates in two distinct phases: (1) \textbf{Simulator Calibration}, where we train a goal-guided user simulator via Supervised Fine-Tuning (SFT) to mimic human interaction patterns; and (2) \textbf{Interactive Policy Optimization}, where the dialogue agent is trained within this calibrated environment using memory-efficient RL to maximize task success.

\subsection{Phase I: Simulator Calibration via Supervised Fine-Tuning (SFT)}

The primary bottleneck in standard interactive RL is the \textit{simulator-induced shift} ($\delta_t$). Conventional prompt-based simulators frequently fail to exhibit the specific intent, cognitive blind spots, or "stubbornness" of real users. To minimize this sim-to-real gap, we formulate the simulator as a trainable agent and explicitly align it with human behavioral distributions using historical multi-turn data.

\textbf{Goal-Guided Simulation Context.} Unlike the dialogue policy, which must infer the user's intent dynamically, the user simulator functions as an environmental "Teacher" and thus requires access to the full ground truth. Let $\mathcal{D} = \{(x, y^*, h_{\text{real}})\}$ denote a dataset containing the initial instruction $x$, the ground-truth solution $y^*$, and the reference human history $h_{\text{real}}$. We condition the simulator $\pi_\phi^{\text{sim}}$ on an oracle context $I_t^{\text{sim}} = [x, y^*, h_t, a_t]$. This privileged information allows the simulator to consistently verify whether the agent’s current action $a_t$ aligns with the target $y^*$ and to provide feedback or corrections consistent with the reference intent, without hallucinating diverging goals.

\textbf{Calibration via Behavioral Cloning.} To ensure the simulator accurately reflects human transition dynamics $\mathcal{T}_{\text{real}}$, we optimize the simulator parameters $\phi$ using Supervised Fine-Tuning (SFT) on the real user turns from $\mathcal{D}$. Specifically, we minimize the negative log-likelihood (NLL) of the reference human responses $u_t$ given the oracle context and the dialogue history up to turn $t$:
\vspace{-1mm}
\begin{equation}
\small
\mathcal{L}_{\text{Calibration}}(\phi) = -\mathbb{E}_{(x,y^*,h_{\text{real}}, a_{\text{real}}) \sim \mathcal{D}} \left[ \sum_{t} \log \pi_{\phi}^{\text{sim}}(u_t \mid x, y^*, h_t, a_t) \right]
\end{equation}
By conditioning on both the ground-truth and the agent's actual output the simulator learns to provide grounded feedback, such as asking for clarification, expressing partial understanding, or persistently rejecting incorrect solutions.

\textbf{Mitigating Simulator Artifacts and Reward Hacking.} The motivation for SFT calibration extends beyond mere behavioral realism; it also provides a critical defense against \textit{reward hacking}\cite{wang2025reward}. Uncalibrated simulators frequently exhibit "sycophancy" (blindly agreeing with flawed logic\cite{sharma2023towards}), which RL agents rapidly exploit by generating confident but incorrect responses to trick the simulator into emitting unearned success signals. By rigorously grounding the simulator in actual human data, we eliminate these exploitable artifacts, ensuring the policy must make genuine progress to achieve high rewards.


\subsection{Phase II: Interactive Policy Optimization}
\label{sec:policy_opt}

With the calibrated simulator $\pi^{\text{sim}}_{\phi^*}$ frozen, we proceed to optimize the dialogue policy $\pi_\theta$. This phase closes the interaction loop to address the \textit{policy-induced shift} ($\Delta_t$).


\textbf{On-Policy Trajectory Generation.} 
In contrast to Static Context RL, our agent generates full multi-turn trajectories by interacting dynamically with the aligned simulator. For an episode of length $H$, the joint trajectory $\tau = (s_0, a_0, u_0, \dots, a_{H-1}, u_{H-1})$ is generated via alternating execution, as summarized in Algorithm~\ref{alg:calibrated_rl}. This ensures that the policy $\pi_\theta$ is trained on the exact state distribution $d^{\pi_\theta}_t$ it dynamically visits, allowing it to naturally develop robust error-recovery and long-horizon planning strategies that are fundamentally unlearnable from static, off-policy logs.

\textbf{Sparse Task Rewards.} 
The quality of the interaction is measured solely by outcome success. We utilize a sparse reward signal $R_{\text{task}}$ administered only at the final turn $H$. Let $\hat{y}$ be the final output produced by the agent (e.g., the refined document or final code):
\begin{equation}
\small
    R(\tau) = \begin{cases} 
    \text{Score}(\hat{y}, y^*) & \text{at } t = H \\
    0 & \text{at } t < H
    \end{cases}
\end{equation}
For document editing tasks, $\text{Score}(\cdot)$ is the BLEU metric; for math and coding tasks, it is a binary correctness indicator. This outcome-based supervision forces the agent to engage in meaningful information seeking. The policy $\pi_\theta$ is then optimized using Group Relative Policy Optimization (GRPO)~\cite{shao2024deepseekmath} to maximize the expected return $\mathbb{E}_{\tau}[R(\tau)]$.


\section{Experiments}
\label{sec:experiments}

We conduct extensive experiments to validate two core hypotheses: (1) \textbf{Interactive RL} significantly outperforms Static Context RL by mitigating the distribution shift between offline training contexts and on-policy generation; and (2) \textbf{Simulator Alignment} is critical for interactive training—optimizing the simulator to mimic human intent and natural interaction patterns leads to superior downstream policy performance.

\subsection{Experimental Setup}

\paragraph{Datasets and Benchmarks.} 
To rigorously evaluate the dynamic interaction capabilities of our framework, we employ two challenging multi-turn dialogue benchmarks inspired by recent collaborative agent evaluations~\cite{wu2025collabllm}. These datasets are explicitly designed to test long-horizon planning, ambiguity resolution, and error recovery:
\begin{itemize}[leftmargin=*]
    \item \textbf{MediumDocEdit-Chat (Collaborative Editing):} Derived from document editing corpora, this task requires the agent to help a user iteratively refine a document. Unlike single-turn instruction following, user intents in this dataset are frequently ambiguous or incomplete. The agent must proactively ask clarifying questions, propose partial edits, and adapt to shifting user feedback over multiple turns. The primary objective evaluation metric is \textbf{BLEU}, calculated between the final model-refined document and the ground-truth target text.
    \item \textbf{MATH-Chat (Conversational Reasoning):} Adapted from the MATH dataset, this benchmark evaluates the agent's ability to act as an interactive tutor. Instead of directly outputting the final answer, the agent must guide a simulated student (who may make logical errors or express confusion) step-by-step toward the correct solution. This requires robust reasoning scaffolding and the ability to correct trajectory deviations dynamically. The task success is measured by the final answer \textbf{Accuracy (Acc.)}.
\end{itemize}

\paragraph{Models and Baselines.}
To ensure a fair and controlled comparison, we utilize \textbf{Gemma-3-4B-IT}~\cite{team2024gemma} as the foundational backbone for both the policy agent and the user simulator across all experiments. We systematically compare our full \textit{Calibrated Interactive RL} framework against established paradigms and internal ablations, seamlessly mirroring the theoretical sources of distribution shift (as summarized in Table~\ref{tab:results}):
\begin{itemize}[leftmargin=*]
    \item \textbf{Gemma-3-4B-IT (Base Model):} The original instruction-tuned model without any multi-turn RL fine-tuning, acting as our zero-shot baseline.
    \item \textbf{\textsc{CollabLLM} (Static RL):} The state-of-the-art offline multi-turn RL method~\cite{wu2025collabllm}. We faithfully reproduce this baseline using the official codebase, optimizing the Gemma-3 policy on static conversational logs using Multi-turn-aware Rewards. This represents the upper bound of the \textit{Static Context RL} paradigm.
    \item \textbf{Static Context:} An ablation of our method where the policy is trained exclusively on fixed offline history distributions ($h \sim d^{\text{data}}$) without dynamic simulator rollouts.
    \item \textbf{Ours (Naive Interactive):} An ablation utilizing the full Interactive RL loop, but employing a basic, uncalibrated user simulator (prompting the Base Qwen2.5-7B-Instruct~\cite{qwen2.5} model without Phase I SFT). This isolates the impact of closing the interaction loop from the simulator's behavioral fidelity.
    \item \textbf{Ours (Calibrated Interactive):} Our complete proposed framework, featuring on-policy trajectory generation against the Phase I SFT-aligned user simulator.
    \item \textbf{Oracle (Proxy Human):} An evaluation upper-bound representing the performance of the base policy when interacting with a privileged proxy human simulator that has direct access to the ground-truth targets.
\end{itemize}

\paragraph{Implementation Details.}
In Phase I, the user simulator (Qwen2.5-7B-Instruct) is trained via SFT on 1{,}860 (MATH) and 16{,}028 (MediumDocEdit) multi-turn trajectories collected from an oracle Qwen3-235B~\cite{yang2025qwen3} simulator. In Phase II, the policy (Gemma-3-4B-IT) is optimized via GRPO on 32 GPUs with a group size of $G=32$ samples per prompt and a global batch size of 8{,}192. We use a constant learning rate of $5 \times 10^{-6}$ with Adam ($\beta_1{=}0.9, \beta_2{=}0.98$) and asymmetric PPO-style clipping~\cite{schulman2017proximal} ($\epsilon \in [0.2, 0.28]$). Full hyperparameter tables are provided in Appendix~\ref{app:implementation}.

\paragraph{Evaluation Metrics.}
Following the evaluation protocol of \textsc{CollabLLM}, we assess performance along two primary dimensions. (1)~\textbf{Task Performance}: Measured by BLEU for editing and Accuracy for math reasoning. (2)~\textbf{Efficiency}: We report the average generation length (\textbf{\#Toks}, in thousands of tokens), where lower values indicate more concise, natural conversational turns rather than monolithic document dumps. We additionally report the \textbf{Average Score} (arithmetic mean of BLEU and ACC) as a unified metric across domains. All experiments are repeated 3 times with different random seeds; we report mean $\pm$ standard deviation.

\providecommand{\std}[1]{{\scriptsize$\pm$#1}}
\begin{table*}[t]
    \centering
    \renewcommand{\arraystretch}{1.25}
    \setlength{\tabcolsep}{7pt}
    \caption{\textbf{Main Results on Multi-turn Collaboration Benchmarks.} We evaluate the Gemma-3-4B-IT policy across two domains: MediumDocEdit (Editing) and MATH-Chat (Reasoning). Each experiment is repeated 3 times; we report mean{\scriptsize$\pm$std}. The blue row highlights our full \textbf{Calibrated Interactive RL} framework.}
    \label{tab:results}
    \resizebox{1.0\textwidth}{!}{
    \begin{tabular}{l l cc cc cc}
        \toprule[1.5pt]
        \multirow{2}{*}{\textbf{Method}}
        & \multirow{2}{*}{\textbf{Paradigm}}
        & \multicolumn{2}{c}{\textbf{MediumDocEdit-Chat}}
        & \multicolumn{2}{c}{\textbf{MATH-Chat}}
        & \multicolumn{2}{c}{\textbf{Average}} \\
        \cmidrule(lr){3-4} \cmidrule(lr){5-6} \cmidrule(lr){7-8}
        & & BLEU $\uparrow$
        & \#Toks\,$(k)\!\downarrow$
        & ACC $\uparrow$
        & \#Toks\,$(k)\!\downarrow$
        & Score $\uparrow$
        & \#Toks\,$(k)\!\downarrow$ \\
        \midrule

        \multicolumn{8}{l}{\textit{\textbf{Baselines}}} \\
        Gemma-3-4B-IT
        & Base Model
        & 32.2\std{1.5} & 2.47\std{0.12}
        & 82.3\std{0.6} & 1.76\std{0.04}
        & 57.3 & 2.12 \\

        \textsc{CollabLLM}$^\dagger$
        & Offline DPO
        & 33.1\std{0.4} & 2.38\std{0.07}
        & 82.3\std{1.9} & 1.63\std{0.09}
        & 57.7 & \textbf{2.01} \\

        \midrule

        \multicolumn{8}{l}{\textit{\textbf{Ours (Gemma-3-4B-IT backbone)}}} \\
        Static Context
        & Static RL
        & 33.8\std{1.8} & 3.09\std{0.01}
        & 85.0\std{0.0} & \textbf{1.63\std{0.07}}
        & 59.4 & 2.36 \\

        Ours (Naive Interactive)
        & Interactive RL
        & 26.1\std{1.2} & 4.00\std{0.07}
        & 89.3\std{3.4} & 1.97\std{0.05}
        & 57.7 & 2.98 \\

        \rowcolor{LightBlue}
        \textbf{Ours (Calibrated Interactive)}
        & \textbf{Calibrated RL}
        & \textbf{34.6\std{1.0}} & \textbf{2.27\std{0.08}}
        & \textbf{91.5\std{1.3}} & 1.86\std{0.09}
        & \textbf{63.05} & 2.06 \\

        \midrule
        \multicolumn{8}{l}{\textit{\textbf{Reference}}} \\
        Oracle (Proxy Human)
        & --
        & 36.2\std{1.4} & 1.83\std{0.06}
        & 89.7\std{1.6} & 1.27\std{0.04}
        & 63.0 & 1.55 \\

        \bottomrule[1.5pt]
    \end{tabular}
    }
    \vspace{2pt}

    {\footnotesize $^\dagger$\textsc{CollabLLM}~\cite{wu2025collabllm} reproduced with offline DPO on the same Gemma-3-4B-IT backbone.}
    \vspace{-10pt}
\end{table*}

\subsection{Experimental Results}

Table~\ref{tab:results} summarizes the comprehensive results across both domains. Our proposed framework achieves the highest MATH-Chat accuracy (\textbf{91.5\%}), substantially outperforming all baselines and even surpassing the Oracle reference (89.7\%). In the following, we provide a detailed analysis of these empirical findings.

\paragraph{Interactive RL Beats Static Context RL.}
Closing the interaction loop significantly improves task success: MATH-Chat accuracy rises from 85.0\% (Static Context) to 89.3\% (Naive Interactive), a gain of +4.3 points. This validates our theoretical analysis (\cref{sec:theory_main}): optimizing the policy on dynamically induced state distributions ($h \sim d^{\pi}$) rather than fixed offline logs equips the agent with robust error-recovery skills and mitigates compounding distribution shifts. Qualitatively, interactive agents successfully rectify logical missteps dynamically, whereas static agents often hallucinate when user feedback deviates from their offline training distribution.
To provide deeper intuition into the behavioral improvements, Figure \ref{fig:case_study} illustrates a qualitative analysis of successfully resolved cases on MATH-Chat. Our Calibrated Interactive RL agent effectively avoids degenerate "try-and-fail" loops, arriving at correct solutions in significantly fewer dialogue turns. We refer readers to Appendix \ref{app:case_study} for comprehensive step-by-step case studies demonstrating the model's enhanced structural understanding and proactive error-recovery strategies.

\paragraph{Calibrated Simulators are Essential for Interaction.}
Merely closing the interaction loop is insufficient; aligning the simulator with human behavior yields further substantial gains. On MATH-Chat, calibrated training achieves \textbf{91.5\%} accuracy---a +2.2 improvement over the Naive Interactive baseline (89.3\%) and notably surpassing even the Oracle (89.7\%). This suggests our aligned simulator provides a highly effective and complementary training signal: by calibrating the simulator to maintain realistic persistence and avoid sycophantic behavior, the policy learns more robust reasoning strategies that generalize even beyond what a single proxy human would elicit.

\paragraph{Comparison with \textsc{CollabLLM}.}
The state-of-the-art static baseline (\textsc{CollabLLM}, offline DPO) performs comparably to the vanilla base model on MATH-Chat (82.3\% vs.\ 82.3\%). This highlights a fundamental limitation of offline RL in multi-turn settings: without on-policy experience, the agent never observes how its responses shape subsequent dialogue turns. The superiority of all our interactive methods confirms that \emph{on-policy experience generation} is a critical prerequisite for robust multi-turn optimization.

\begin{figure}[t]
    \centering
    \includegraphics[width=\linewidth]{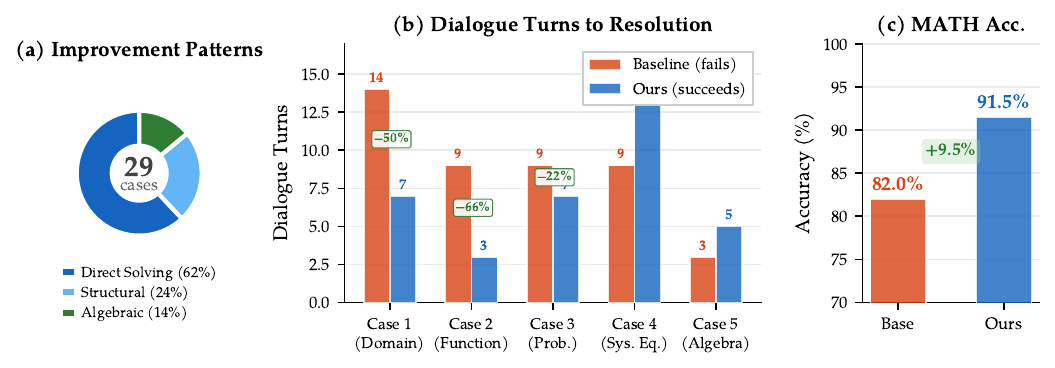}
    \caption{\textbf{Qualitative analysis of improved cases on MATH-Chat.} (a) Distribution of improvement patterns across 29 cases where the baseline fails but our method succeeds. (b) Dialogue turns required to resolve each case---our model achieves correct answers in up to 67\% fewer turns. (c) Overall accuracy improvement: 82.3\% $\rightarrow$ 91.5\%. Detailed case studies are provided in Appendix~\ref{app:case_study}. Note: While Table \ref{tab:results} reports the mean and standard deviation across three independent runs, the qualitative patterns and case studies presented here are derived from a single representative trial to maintain trajectory consistency.}
    \label{fig:case_study}
    \vspace{-5mm}
\end{figure}

\section{Conclusion}
\label{sec:conclusion}

In this work, we established the theoretical foundations for the limitations of Static Context RL, formally proving that training on fixed offline logs induces a context distribution shift where errors compound quadratically over turns. To address this fundamental bottleneck, we proposed \textbf{Calibrated Interactive RL}, a unified framework that couples interactive training with rigorous simulator alignment. We demonstrated that simply closing the interaction loop is insufficient; the success of interactive learning critically depends on minimizing the sim-to-real gap. By treating the user simulator as a trainable agent and optimizing it to maximize a discriminator-derived realism reward, we effectively aligned the simulation environment with human interaction patterns. Our extensive experiments confirm that our framework not only mitigates policy-induced distribution shifts but also fosters robust information-seeking behaviors, significantly outperforming baselines relying on static contexts or uncalibrated simulators. This work highlights the necessity of high-fidelity, aligned simulation as a prerequisite for the next generation of interactive dialogue agents.
\nocite{langley00}

\newpage


\bibliography{example_paper}
\bibliographystyle{plainnat}

\appendix
\onecolumn

\section{Theoretical Proofs and Derivations}
\label{app:theory-proofs}

In this appendix, we provide full proofs for the theoretical claims presented in \cref{sec:theory_main}. We first restate the necessary notation and assumptions, then establish auxiliary lemmas regarding distributional divergence, and finally derive the bounds for Policy Mismatch (\cref{thm:policy_shift}) and the Sim-to-Real Gap (\cref{thm:sim_gap}).

\subsection{Notation and Preliminaries}
\label{app:notation}

We consider a finite-horizon Markov Decision Process (MDP) defined over histories to account for the non-Markovian nature of dialogue when viewed from the perspective of state-only transitions.
\begin{itemize}
    \item \textbf{Horizon:} $H \in \mathbb{N}$. Time steps are indexed by $t = 0, \dots, H-1$.
    \item \textbf{History:} Let $h_t = (s_0, a_0, \dots, s_t) \in \mathcal{H}$ denote the interaction history up to time $t$.
    \item \textbf{Distributions:} Let $d^\pi_t(h)$ denote the marginal distribution over histories at time $t$ induced by policy $\pi$ and the environment dynamics.
    \item \textbf{Norms:} Unless otherwise specified, $\|\cdot\|_1$ denotes the $\ell_1$-norm (total variation distance) for probability distributions. For a function $f$, $\|f\|_\infty = \sup_x |f(x)|$.
    \item \textbf{Rewards:} We assume a bounded reward function $r(h, a)$ such that $|r(h, a)| \le R_{\max}$ for all $h, a$.
\end{itemize}

\subsection{Auxiliary Lemmas}
We begin by establishing a decomposition lemma that bounds the divergence of a joint distribution in terms of the marginal distribution divergence and the conditional distribution divergence.

\begin{lemma}[Joint Distribution Decomposition]
\label{lemma:joint_decomp}
Consider two joint distributions $p(x, y) = p(x)p(y|x)$ and $q(x, y) = q(x)q(y|x)$ over discrete domains $\mathcal{X}, \mathcal{Y}$. Then:
\begin{equation}
    \|p - q\|_1 \le \|p(x) - q(x)\|_1 + \mathbb{E}_{x \sim q(x)} [\|p(y|x) - q(y|x)\|_1].
\end{equation}
\end{lemma}

\begin{proof}
By definition of the $\ell_1$-norm:
\begin{align*}
    \|p - q\|_1 &= \sum_{x,y} |p(x)p(y|x) - q(x)q(y|x)| \\
    &= \sum_{x,y} |p(x)p(y|x) - q(x)p(y|x) + q(x)p(y|x) - q(x)q(y|x)| \\
    &\le \sum_{x,y} |p(x) - q(x)|p(y|x) + \sum_{x,y} q(x)|p(y|x) - q(y|x)| \\
    &= \sum_{x} |p(x) - q(x)| \underbrace{\sum_{y} p(y|x)}_{=1} + \sum_{x} q(x) \sum_{y} |p(y|x) - q(y|x)| \\
    &= \|p(x) - q(x)\|_1 + \mathbb{E}_{x \sim q(x)} [\|p(y|x) - q(y|x)\|_1].
\end{align*}
\end{proof}

\begin{lemma}[Error Propagation Dynamics]
\label{lemma:error_prop}
Let $\mathcal{T}$ be the environment transition operator mapping the joint distribution at time $t$ to the marginal history distribution at $t+1$. Let $\Delta_t = \|d^\pi_t - d^{\pi'}_t\|_1$ and $\varepsilon_t = \mathbb{E}_{h \sim d^{\pi'}_t} \|\pi(\cdot|h) - \pi'(\cdot|h)\|_1$. Assuming the start state distribution is fixed ($\Delta_0 = 0$), then for all $t \ge 0$:
\begin{equation}
    \Delta_{t+1} \le \Delta_t + \varepsilon_t \quad \implies \quad \Delta_t \le \sum_{i=0}^{t-1} \varepsilon_i.
\end{equation}
\end{lemma}

\begin{proof}
Let $q^\pi_t(h, a) = d^\pi_t(h)\pi(a|h)$ be the joint distribution at time $t$. The transition to the next history $h_{t+1}$ is deterministic given $h_t, a_t$ and the stochastic environment transition $s_{t+1} \sim P(\cdot|s_t, a_t)$. The operator mapping $q_t \to d_{t+1}$ is a linear map represented by a stochastic matrix, making it non-expansive in the $\ell_1$-norm. Thus:
\begin{equation*}
    \Delta_{t+1} = \|d^\pi_{t+1} - d^{\pi'}_{t+1}\|_1 \le \|q^\pi_t - q^{\pi'}_t\|_1.
\end{equation*}
Applying Lemma~\ref{lemma:joint_decomp} to the RHS:
\begin{equation*}
    \|q^\pi_t - q^{\pi'}_t\|_1 \le \|d^\pi_t - d^{\pi'}_t\|_1 + \mathbb{E}_{h \sim d^{\pi'}_t} \|\pi(\cdot|h) - \pi'(\cdot|h)\|_1 = \Delta_t + \varepsilon_t.
\end{equation*}
Combining these gives the recurrence $\Delta_{t+1} \le \Delta_t + \varepsilon_t$. Telescoping the sum from $i=0$ to $t-1$ with $\Delta_0=0$ yields the result.
\end{proof}

\subsection{Proof of Theorem~\ref{thm:policy_shift} (Performance Difference via Policy Shift)}

\textbf{Theorem~\ref{thm:policy_shift}.} \textit{For any two policies $\pi, \pi'$ over finite horizon $H$, the difference in expected return is bounded by:}
\begin{equation*}
    |J(\pi) - J(\pi')| \le R_{\max} \sum_{i=0}^{H-1} (H - i) \varepsilon_i.
\end{equation*}

\begin{proof}
The difference in expected return can be written as the sum of differences in expected rewards at each time step. Let $q^\pi_t$ denote the joint distribution of $(h, a)$ at time $t$ under policy $\pi$.
\begin{align*}
    |J(\pi) - J(\pi')| &= \left| \sum_{t=0}^{H-1} \left( \mathbb{E}_{(h,a) \sim q^\pi_t}[r(h,a)] - \mathbb{E}_{(h,a) \sim q^{\pi'}_t}[r(h,a)] \right) \right| \\
    &\le \sum_{t=0}^{H-1} \left| \sum_{h,a} (q^\pi_t(h,a) - q^{\pi'}_t(h,a))r(h,a) \right|.
\end{align*}
Using Hölder's inequality (bounding the expectation difference by the $\ell_1$ distance and max reward):
\begin{equation*}
    \left| \sum_{h,a} (q^\pi_t(h,a) - q^{\pi'}_t(h,a))r(h,a) \right| \le R_{\max} \|q^\pi_t - q^{\pi'}_t\|_1.
\end{equation*}
From the proof of Lemma~\ref{lemma:error_prop}, we know $\|q^\pi_t - q^{\pi'}_t\|_1 \le \Delta_t + \varepsilon_t$. Thus:
\begin{equation*}
    |J(\pi) - J(\pi')| \le R_{\max} \sum_{t=0}^{H-1} (\Delta_t + \varepsilon_t).
\end{equation*}
Substituting $\Delta_t \le \sum_{j=0}^{t-1} \varepsilon_j$:
\begin{equation*}
    \sum_{t=0}^{H-1} (\Delta_t + \varepsilon_t) = \sum_{t=0}^{H-1} \left( \sum_{j=0}^{t-1} \varepsilon_j + \varepsilon_t \right) = \sum_{t=0}^{H-1} \sum_{j=0}^{t} \varepsilon_j.
\end{equation*}
By changing the order of summation, the term $\varepsilon_i$ appears for every $t \ge i$:
\begin{equation*}
    \sum_{t=0}^{H-1} \sum_{j=0}^{t} \varepsilon_j = \sum_{i=0}^{H-1} (H - i)\varepsilon_i.
\end{equation*}
Multiplying by $R_{\max}$ concludes the proof.
\end{proof}

\subsection{Proof of Theorem~\ref{thm:sim_gap} (Sim-to-Real Gap)}

\textbf{Theorem~\ref{thm:sim_gap}.} \textit{For any fixed policy $\pi$, let $\delta_t = \mathbb{E}_{(h,a) \sim q^{\text{real}}_t} \|\mathcal{T}_{\text{real}}(\cdot|h, a) - \mathcal{T}_{\text{sim}}(\cdot|h, a)\|_1$. Then:}
\begin{equation*}
    |J_{\text{real}}(\pi) - J_{\text{sim}}(\pi)| \le R_{\max} \sum_{i=0}^{H-1} (H - i) \delta_i.
\end{equation*}

\begin{proof}
Let $d^{\text{real}}_t$ and $d^{\text{sim}}_t$ denote the history distributions induced by $\pi$ in the real and simulated environments, respectively. Let $\Delta^{\text{env}}_t = \|d^{\text{real}}_t - d^{\text{sim}}_t\|_1$.

Since the policy $\pi$ is fixed in both environments, the joint distributions satisfy $q^{\text{real}}_t(h,a) = d^{\text{real}}_t(h)\pi(a|h)$ and $q^{\text{sim}}_t(h,a) = d^{\text{sim}}_t(h)\pi(a|h)$, so $\|q^{\text{real}}_t - q^{\text{sim}}_t\|_1 = \Delta^{\text{env}}_t$.

The value difference is bounded by the cumulative divergence in joint distributions plus the local transition errors:
\begin{equation*}
    |J_{\text{real}}(\pi) - J_{\text{sim}}(\pi)| \le R_{\max} \sum_{t=0}^{H-1} (\Delta^{\text{env}}_t + \delta_t).
\end{equation*}
This follows because at each step $t$, the reward difference depends on the state at $t+1$, which is influenced by both the history mismatch $\Delta^{\text{env}}_t$ and the current transition error $\delta_t$.

We now analyze the recurrence of $\Delta^{\text{env}}_t$. The history distribution at $t+1$ is obtained by applying the respective transition operators.
\begin{align*}
    \Delta^{\text{env}}_{t+1} &= \|\mathcal{T}_{\text{real}}(q^{\text{real}}_t) - \mathcal{T}_{\text{sim}}(q^{\text{sim}}_t)\|_1 \\
    &= \|\mathcal{T}_{\text{real}}(q^{\text{real}}_t) - \mathcal{T}_{\text{sim}}(q^{\text{real}}_t) + \mathcal{T}_{\text{sim}}(q^{\text{real}}_t) - \mathcal{T}_{\text{sim}}(q^{\text{sim}}_t)\|_1 \\
    &\le \underbrace{\|\mathcal{T}_{\text{real}}(q^{\text{real}}_t) - \mathcal{T}_{\text{sim}}(q^{\text{real}}_t)\|_1}_{(A)} + \underbrace{\|\mathcal{T}_{\text{sim}}(q^{\text{real}}_t) - \mathcal{T}_{\text{sim}}(q^{\text{sim}}_t)\|_1}_{(B)}.
\end{align*}

\textbf{Term (A) - Local Simulator Error:} This measures the error of the simulator on the \textit{real} distribution.
\begin{align*}
    (A) &= \sum_{h', h, a} |P_{\text{real}}(h'|h, a)q^{\text{real}}_t(h, a) - P_{\text{sim}}(h'|h, a)q^{\text{real}}_t(h, a)| \\
    &= \sum_{h, a} q^{\text{real}}_t(h, a) \|P_{\text{real}}(\cdot|h, a) - P_{\text{sim}}(\cdot|h, a)\|_1 \\
    &= \mathbb{E}_{(h,a) \sim q^{\text{real}}_t} [\|\mathcal{T}_{\text{real}}(\cdot|h, a) - \mathcal{T}_{\text{sim}}(\cdot|h, a)\|_1] = \delta_t.
\end{align*}

\textbf{Term (B) - Error Propagation:} Since $\mathcal{T}_{\text{sim}}$ is a stochastic map, it is non-expansive.
\begin{align*}
    (B) &\le \|q^{\text{real}}_t - q^{\text{sim}}_t\|_1 \\
    &= \sum_{h,a} |d^{\text{real}}_t(h)\pi(a|h) - d^{\text{sim}}_t(h)\pi(a|h)| \\
    &= \sum_{h} |d^{\text{real}}_t(h) - d^{\text{sim}}_t(h)| \underbrace{\sum_{a} \pi(a|h)}_{1} = \Delta^{\text{env}}_t.
\end{align*}

Thus, we have the recurrence $\Delta^{\text{env}}_{t+1} \le \Delta^{\text{env}}_t + \delta_t$. Assuming $\Delta^{\text{env}}_0 = 0$, this unfolds to $\Delta^{\text{env}}_t \le \sum_{i=0}^{t-1} \delta_i$. Substituting this back:
\begin{equation*}
    \sum_{t=0}^{H-1} (\Delta^{\text{env}}_t + \delta_t) = \sum_{t=0}^{H-1}\left(\sum_{i=0}^{t-1}\delta_i + \delta_t\right) = \sum_{t=0}^{H-1}\sum_{j=0}^{t}\delta_j = \sum_{i=0}^{H-1}(H-i)\delta_i.
\end{equation*}
Multiplying by $R_{\max}$ concludes the proof.
\end{proof}

\section{Experimental Details}
\label{app:experimental_details}

In this section, we provide a comprehensive description of the datasets, model configurations, training hyperparameters, and computational infrastructure used in our experiments. Our implementation builds on the SLIME distributed RL framework and the CollabLLM evaluation pipeline.

\subsection{Datasets and Benchmarks}

We evaluate our proposed framework on two distinct multi-turn dialogue domains that require long-horizon planning and precise information seeking.

\paragraph{MATH-Chat.}
This benchmark evaluates the reasoning and tutoring capability of the agent in a conversational math problem-solving scenario.
\begin{itemize}[leftmargin=*, nosep]
    \item \textbf{Data Source:} We adapt problems from the MATH dataset~\cite{hendrycks2021measuring}, specifically selecting \textit{Hard}-level problems that require multi-step reasoning.
    \item \textbf{Task Setup:} For each math problem, the ground-truth solution serves as $y^*$. The user simulator plays the role of a student who may express confusion, make logical errors, or ask for clarification, requiring the agent to guide them step-by-step rather than directly revealing the answer.
    \item \textbf{Statistics:} The RL training set comprises ${\sim}$3{,}000 problems (stored in Parquet format). Evaluation is conducted on a held-out set of 200 problems, with results averaged over 3 independent runs.
    \item \textbf{Reward:} Binary accuracy---an LLM judge (Qwen3-235B) determines whether the agent's final answer matches the ground truth ($R \in \{0, 1\}$).
\end{itemize}

\paragraph{MediumDocEdit-Chat.}
This dataset focuses on collaborative document editing, requiring the agent to iteratively refine a document based on user feedback.
\begin{itemize}[leftmargin=*, nosep]
    \item \textbf{Data Source:} Derived from article writing and editing corpora. Each sample consists of an initial user request $x$ and a ground-truth refined document $y^*$.
    \item \textbf{Task Setup:} The user provides ambiguous or incomplete editing instructions, requiring the agent to proactively ask clarifying questions and propose partial edits across multiple turns.
    \item \textbf{Statistics:} The RL training set contains ${\sim}$38{,}118 editing prompts (Parquet format). Evaluation uses a held-out set of 200 samples, averaged over 3 runs.
    \item \textbf{Reward:} Sentence-level BLEU score between the agent's final produced document and the ground-truth reference $y^*$ ($R \in [0, 1]$).
\end{itemize}

\subsection{Model Architectures}

\paragraph{Policy Agent ($\pi_\theta$).}
We use \textbf{Gemma-3-4B-IT}~\cite{team2024gemma} as the backbone for the dialogue policy across all experiments. This model is chosen for its strong instruction-following capabilities at a compact 4B parameter scale, which enables efficient on-policy rollout generation during interactive RL training. The model is loaded in \texttt{bfloat16} precision with Flash Attention 2 enabled.

\paragraph{User Simulator ($\pi_\phi^{\text{sim}}$).}
We use \textbf{Qwen2.5-7B-Instruct}~\cite{qwen2.5} as the simulator backbone. The larger 7B scale (relative to the 4B policy) ensures the simulator has sufficient capacity to model diverse and realistic user behaviors. After Phase I SFT calibration, the simulator is frozen for Phase II training.

\paragraph{Evaluation Simulator \& Judge.}
During evaluation, we employ \textbf{Qwen3-235B-A22B-Instruct} (FP8 quantized) as both the user simulator and the accuracy/quality judge. This ensures evaluation is \textit{independent} of the training-time simulator, providing an unbiased assessment of the policy's generalization.

\subsection{Training Phase I: Simulator Calibration (SFT)}

In the first phase, we calibrate the user simulator via Supervised Fine-Tuning (SFT) on multi-turn interaction traces collected from the Qwen3-235B oracle simulator interacting with the base policy model. This is a behavioral cloning approach---the simulator learns to replicate realistic human-like interaction patterns.

\paragraph{Data Collection.}
We deploy the Qwen3-235B model as an oracle user simulator (conditioned on ground-truth goals) and collect multi-turn trajectories by having it interact with the Gemma-3-4B-IT base model. This produces:
\begin{itemize}[leftmargin=*, nosep]
    \item \textbf{MATH-Chat:} 1{,}860 multi-turn trajectories (max sequence length: 4{,}096 tokens).
    \item \textbf{MediumDocEdit-Chat:} 16{,}028 multi-turn trajectories (max sequence length: 6{,}144 tokens).
\end{itemize}

\paragraph{Training Setup.}
We fine-tune Qwen2.5-7B-Instruct using standard causal language modeling loss on the user turns only (masking assistant turns). Training uses FSDP on $8\times$ NVIDIA A800-80GB GPUs. Hyperparameters are detailed in Appendix~\ref{app:implementation}, Table~\ref{tab:sft_hyperparams}.

\subsection{Training Phase II: Interactive Policy Optimization (GRPO)}

In the second phase, the SFT-calibrated simulator is frozen, and we optimize the policy agent $\pi_\theta$ via Group Relative Policy Optimization (GRPO)~\cite{shao2024deepseekmath}. GRPO eliminates the need for a separate value network by normalizing rewards within a group of sampled trajectories, making it memory-efficient for long-context multi-turn training.

\paragraph{Interaction Protocol.}
For each training prompt $x$, the agent generates $G=32$ independent multi-turn trajectories by interacting with the aligned simulator. Each episode runs for up to $H=7$ dialogue rounds (14 total turns). The simulator generates the first user message and the episode terminates upon the simulator emitting a \texttt{[[TERMINATE CHAT]]} signal or reaching the maximum turn count.

\paragraph{Reward Function.}
We use sparse outcome rewards $R(\tau)$ administered only at episode completion:
\begin{itemize}[leftmargin=*, nosep]
    \item \textbf{MATH-Chat:} Binary accuracy (0/1) determined by an LLM judge comparing the final answer to ground truth.
    \item \textbf{MediumDocEdit-Chat:} BLEU score between the final document and the reference.
\end{itemize}

\paragraph{Training Setup.}
We train with Ray-based distributed FSDP on 32 GPUs (4 nodes $\times$ 8 A800-80GB) with SGLang for efficient rollout generation. The policy and rollout engine are co-located on the same GPUs (\texttt{--colocate} mode). Full hyperparameters are in Appendix~\ref{app:implementation}, Table~\ref{tab:grpo_hyperparams}.

\subsection{Baselines Implementation}

\paragraph{Gemma-3-4B-IT (Base Model).} The original instruction-tuned model without any multi-turn RL fine-tuning, serving as the zero-shot baseline. We evaluate it with the same proactive system prompt used during RL training.

\paragraph{CollabLLM (Static RL).}
We reproduce CollabLLM~\cite{wu2025collabllm} using the official codebase. CollabLLM represents the state-of-the-art in \textit{Static Context RL}: it trains the policy on fixed offline conversation logs using Multi-turn-aware Rewards (MR). To ensure a fair comparison, we adapt CollabLLM to use Gemma-3-4B-IT as the backbone (matching our method), rather than the models used in their original paper. The reward model and training procedure follow their official implementation.

\paragraph{Ours (Static Context).}
An ablation of our framework where the policy is trained via GRPO on static offline trajectories \textit{without} dynamic simulator interaction. The policy optimizes against fixed conversation logs collected from the oracle simulator, conditioning on offline history $h \sim d^{\text{data}}_t$ rather than self-generated trajectories. This isolates the effect of closing the interaction loop.

\paragraph{Ours (Naive Interactive).}
An ablation using the full interactive RL loop but with an uncalibrated user simulator. Specifically, we use the base Qwen2.5-7B-Instruct model (without Phase I SFT) prompted to act as a user. This isolates the impact of simulator calibration from the benefit of interactive training.

\subsection{Computational Infrastructure}

\begin{itemize}[leftmargin=*, nosep]
    \item \textbf{Hardware:} All experiments were conducted on NVIDIA A800-80GB GPUs. Phase I uses 8 GPUs; Phase II uses 32 GPUs (4 nodes).
    \item \textbf{Software:} Python 3.12, PyTorch 2.x, SLIME (distributed RL framework), SGLang (efficient LLM serving for rollout generation), and the CollabLLM evaluation pipeline.
    \item \textbf{Compute Cost:} Phase I (SFT) completes in $<$1 hour per task. Phase II (GRPO) trains for 10 steps (MATH) / 20 steps (MediumDocEdit), taking approximately 4--8 hours per task on 32 GPUs.
    \item \textbf{Networking:} NCCL communication with InfiniBand disabled (\texttt{NCCL\_IB\_DISABLE=1}); extended timeouts (72{,}000s) for long rollout steps.
\end{itemize}

\section{Prompts}
\label{app:prompts}

This section presents the exact prompts used in our framework. We include: (1) the \textbf{User Simulator} prompt used during both Phase I calibration and Phase II interactive rollouts, (2) the \textbf{Policy Agent} (proactive assistant) instruction used during RL training, and (3) the \textbf{Accuracy Judge} prompt used for reward computation. Template variables populated at runtime are highlighted in \templatevar{blue}.

\subsection{User Simulator Prompt}

The user simulator is conditioned on a \textit{task description} and the \textit{ground-truth goal} (reference answer), but is instructed to behave as a realistic human user---releasing information progressively, making occasional mistakes, and minimizing effort. This prompt is used both for the Qwen3-235B oracle (whose outputs serve as SFT targets in Phase I) and for the SFT-trained Qwen2.5-7B simulator during Phase II rollouts.

\begin{promptbox}{User Simulator Prompt}
You are role-playing as a human USER interacting with an AI collaborator to complete a specific task. Your goal is to generate realistic, natural responses that a user might give in this scenario.

\textbf{Input Information:}\\
You will be provided with:
\begin{itemize}[leftmargin=*, nosep]
    \item Task Description: The type of task you are trying to accomplish.
    \item Complete Prompt or Reference Goal: This field may include the complete user request/query or a reference answer. Use this to understand the user's intent and what would count as a satisfactory outcome.
    \item Chat History: The ongoing conversation between you (as the user) and the AI.
\end{itemize}

\textbf{Inputs:}\\
\texttt{<|The Start of Task Description (Not visible to the AI)|>}\\
\templatevar{task\_desc}\\
\texttt{<|The End of Task Description|>}

\texttt{<|The Start of Complete Prompt or Reference Goal (Not visible to the AI)|>}\\
\templatevar{single\_turn\_prompt}\\
\texttt{<|The End of Complete Prompt or Reference Goal|>}

\texttt{<|The Start of Chat History|>}\\
\templatevar{chat\_history}\\
\texttt{<|The End of Chat History|>}

\textbf{Guidelines:}
\begin{itemize}[leftmargin=*, nosep]
    \item \textbf{Stay in Character:} Role-play as a human USER. You are NOT an AI. Maintain a consistent personality throughout the chat.
    \item \textbf{Minimize Effort:} IMPORTANT! Avoid being too detailed. Provide vague or incomplete demands in early stages. Let the AI ask for clarification rather than providing everything upfront.
    \item \textbf{Knowledge Background:} Reflect the user's knowledge level. If less knowledgeable, you might not notice incorrect statements. Ask questions demonstrating your current understanding and areas of confusion.
    \item \textbf{Occasionally Make Mistakes:} Real users might misspell words, give wrong information, or ask unclear questions. Simulate this behavior.
    \item \textbf{Goal-Oriented:} Keep the chat focused on your intent. Avoid small talk or digressions.
\end{itemize}

\textbf{Output Format:}\\
Output a JSON object with three entries:
\begin{itemize}[leftmargin=*, nosep]
    \item \texttt{"current\_answer"} (str): Briefly summarize the AI's current solution.
    \item \texttt{"thought"} (str): Your thought process deciding what to say next.
    \item \texttt{"response"} (str): Your response to the AI as the user you are role-playing.
\end{itemize}

\textbf{Important Notes:}
\begin{itemize}[leftmargin=*, nosep]
    \item If ``Chat History'' is empty, start the conversation from scratch with an initial request.
    \item Use the provided information for understanding context only. Avoid copying target queries directly.
    \item Use ``\templatevar{terminal\_signal}'' as your response when your goal has been solved or if you determine the AI cannot help further.
\end{itemize}
\end{promptbox}

\subsection{Policy Agent Prompt (Proactive Assistant)}

The policy model receives this system-level instruction that encourages interactive, collaborative behavior. This prompt is used during both RL training and evaluation to guide the agent toward asking clarifying questions and providing structured assistance.

\begin{promptbox}{Policy Agent System Prompt}
You are an AI assistant interacting with a user to perform tasks such as writing, analysis, question answering, math, and coding. Your goal is to generate a response to the user's last message. You should be helpful, collaborative, and highly interactive.

\textbf{Input:}\\
\texttt{<|The Start of Conversation History|>}\\
\templatevar{chat\_history}\\
\texttt{<|The End of Conversation History|>}

\texttt{<|The Start of Additional Information|>}\\
\templatevar{additional\_info}\\
\texttt{<|The End of Additional Information|>}

\textbf{Guidelines:}
\begin{enumerate}[leftmargin=*, nosep]
    \item \textbf{Understanding \& Engagement} --- Accurately interpret the user's intent. Acknowledge previous interactions to maintain continuity.
    \item \textbf{Interactivity (Important!)} --- Ask clarifying questions if the request lacks detail. Ask \textit{specific} follow-up questions (e.g., ``Would you like more information on X?'' not ``Let me know if this is helpful.''). Collaboratively offer guidance in complex situations. Focus on the user's long-term goals.
    \item \textbf{Efficiency} --- Be mindful of how much the user needs to read or type. When presenting options, provide multiple-choice suggestions to make it easier to respond quickly. Avoid repeating information unnecessarily.
    \item \textbf{Communication Style} --- Be honest. If unsure, say ``I don't know'' and suggest alternatives. Align tone with the user's emotional state.
\end{enumerate}

\textbf{Output Format:}\\
Output a JSON object:
\begin{itemize}[leftmargin=*, nosep]
    \item \texttt{"current\_problem"} (str): What is the user's current problem or confusion?
    \item \texttt{"thought"} (str): Your reasoning about what to say next.
    \item \texttt{"response"} (str): Your response following the guidelines above.
\end{itemize}
\end{promptbox}

\subsection{Accuracy Judge Prompt (Reward Function)}

During Phase II training, we use an LLM-based accuracy judge as the reward function for MATH-Chat. The judge evaluates whether the policy's final answer matches the ground truth. This prompt is sent to Qwen3-235B-Instruct.

\begin{promptbox}{Accuracy Judge Prompt}
You are a helpful and meticulous evaluator. Your task is to evaluate the \textit{accuracy} of an AI model's answer to a target question. You will be given the target question, the ground truth answer, and the model's response.

\textbf{Provided Information:}

\texttt{<|The Start of Target Question and Ground Truth Answer|>}\\
Target Question: \templatevar{single\_turn\_prompt}\\
Ground Truth Answer: \templatevar{groundtruth}\\
\texttt{<|The End of Target Question and Ground Truth Answer|>}

\texttt{<|The Start of The Model's Response|>}\\
\templatevar{completion}\\
\texttt{<|The End of The Model's Response|>}

You should determine whether the model's final response to the target question is factually correct and consistent with the provided ground truth.

\textbf{Rating criteria (binary):}
\begin{itemize}[leftmargin=*, nosep]
    \item 1 = Correct --- the response matches the ground truth.
    \item 0 = Incorrect --- the response contradicts or misses the ground truth.
\end{itemize}

\textbf{Output format (JSON):}\\
\texttt{\{\ "thought": "<your reasoning here>",\ "accuracy": <0 or 1>\ \}}
\end{promptbox}
\section{Implementation Details}
\label{app:implementation}

This appendix provides full implementation details for reproducibility. Our framework consists of two training phases: Phase I (Simulator Calibration via SFT) and Phase II (Interactive Policy Optimization via GRPO).

\subsection{Phase I: Simulator Calibration (SFT)}

\paragraph{Base Model.} We use \textbf{Qwen2.5-7B-Instruct} as the backbone for the user simulator. This model is chosen for its strong instruction-following capabilities and multilingual performance, providing a robust foundation for behavioral cloning.

\paragraph{Datasets.} We collect multi-turn interaction traces from a high-fidelity oracle simulator (Qwen3-235B-A22B-Instruct) interacting with the base policy model. These traces capture realistic user behaviors including clarification requests, partial understanding expressions, and persistent goal pursuit. The resulting training sets are:
\begin{itemize}[leftmargin=*]
    \item \textbf{MATH-Chat Simulator:} 1{,}860 multi-turn trajectories, each containing the problem statement, ground-truth solution, and the oracle user's turns. Maximum sequence length: 4{,}096 tokens.
    \item \textbf{MediumDocEdit-Chat Simulator:} 16{,}028 multi-turn trajectories capturing collaborative editing interactions. Maximum sequence length: 6{,}144 tokens.
\end{itemize}

\paragraph{Training Configuration.} We train with FSDP (Fully Sharded Data Parallelism) on 8$\times$ A800-80GB GPUs. The detailed hyperparameters are listed in Table~\ref{tab:sft_hyperparams}.

\begin{table}[h]
\centering
\caption{Phase I: SFT Hyperparameters for Simulator Calibration.}
\label{tab:sft_hyperparams}
\vspace{0.1cm}
\small
\begin{tabular}{lcc}
\toprule
\textbf{Hyperparameter} & \textbf{MATH-Chat} & \textbf{MediumDocEdit} \\
\midrule
Base Model & \multicolumn{2}{c}{Qwen2.5-7B-Instruct} \\
Training Samples & 1{,}860 & 16{,}028 \\
Max Sequence Length & 4{,}096 & 6{,}144 \\
Number of Epochs & 2 & 3 \\
Per-device Batch Size & 2 & 1 \\
Gradient Accumulation & 4 & 8 \\
Effective Batch Size & 64 & 64 \\
Learning Rate & $1 \times 10^{-5}$ & $2 \times 10^{-5}$ \\
LR Scheduler & Cosine & Cosine \\
Warmup Ratio & 0.1 & 0.03 \\
Weight Decay & 0.01 & 0.01 \\
Precision & BFloat16 & BFloat16 \\
Attention & Flash Attention 2 & Flash Attention 2 \\
Seed & 42 & 42 \\
GPUs & 8$\times$ A800-80GB & 8$\times$ A800-80GB \\
\bottomrule
\end{tabular}
\end{table}

\subsection{Phase II: Interactive Policy Optimization (GRPO)}

\paragraph{Base Model.} We use \textbf{Gemma-3-4B-IT} as the policy model backbone. This compact yet capable model enables efficient on-policy rollout generation during interactive RL training.

\paragraph{Datasets.} The RL training prompts consist of:
\begin{itemize}[leftmargin=*]
    \item \textbf{MATH-Chat:} ${\sim}$3{,}000 math problems (from MATH-Hard) formatted as initial user queries in Parquet format.
    \item \textbf{MediumDocEdit-Chat:} ${\sim}$38{,}118 article writing/editing prompts in Parquet format.
\end{itemize}

\paragraph{RL Algorithm.} We use Group Relative Policy Optimization (GRPO) with the following key design choices:
\begin{itemize}[leftmargin=*]
    \item \textbf{Advantage Estimation:} Group-level normalization: $\hat{A}(\tau) = (R(\tau) - \mu_G) / \sigma_G$, where $G$ is the group of trajectories sampled for the same prompt.
    \item \textbf{Clipping:} Asymmetric clipping with $\epsilon_{\text{low}} = 0.2$ and $\epsilon_{\text{high}} = 0.28$.
    \item \textbf{KL Regularization:} Low-variance KL loss with coefficient $\beta = 0.0$ (effectively disabled, relying on clipping for stability).
    \item \textbf{Reward:} Binary accuracy oracle for MATH-Chat; BLEU score for MediumDocEdit-Chat.
\end{itemize}

\paragraph{Training Configuration.} We train with Ray-based distributed FSDP on 32 GPUs (4 nodes $\times$ 8 GPUs) with SGLang for efficient rollout generation. The detailed hyperparameters are listed in Table~\ref{tab:grpo_hyperparams}.

\begin{table}[h]
\centering
\caption{Phase II: GRPO Hyperparameters for Policy Optimization.}
\label{tab:grpo_hyperparams}
\vspace{0.1cm}
\small
\begin{tabular}{lcc}
\toprule
\textbf{Hyperparameter} & \textbf{MATH-Chat} & \textbf{MediumDocEdit} \\
\midrule
Base Model & \multicolumn{2}{c}{Gemma-3-4B-IT} \\
Number of GPUs & 32 (4 nodes) & 32 (4 nodes) \\
Samples per Prompt ($G$) & 32 & 32 \\
Global Batch Size & 8{,}192 & 8{,}192 \\
Rollout Batch Size & 256 & 256 \\
Max Response Length & 8{,}192 & 8{,}192 \\
Rollout Temperature & 1.0 & 1.0 \\
Learning Rate & $5 \times 10^{-6}$ & $5 \times 10^{-6}$ \\
LR Scheduler & Constant & Constant \\
Optimizer & Adam & Adam \\
Adam $\beta_1$, $\beta_2$ & 0.9, 0.98 & 0.9, 0.98 \\
Weight Decay & 0.01 & 0.01 \\
Clip $\epsilon$ (low/high) & 0.2 / 0.28 & 0.2 / 0.28 \\
KL Coefficient & 0.0 & 0.0 \\
Gradient Checkpointing & \checkmark & \checkmark \\
Max Tokens/GPU & 1{,}024 & 1{,}024 \\
SGLang TP per Engine & 2 & 2 \\
Save Interval (steps) & 1 & 1 \\
Training Steps & 10 & 20 \\
\bottomrule
\end{tabular}
\end{table}

\paragraph{Evaluation Setup.} All evaluations use \textbf{Qwen3-235B-A22B-Instruct} as both the user simulator and the evaluation judge during testing. This ensures that the evaluation is independent of the training-time simulator. Key evaluation parameters:
\begin{itemize}[leftmargin=*]
    \item Evaluation set size: 200 problems per domain
    \item Maximum dialogue turns: 14 (7 rounds of user-assistant exchange)
    \item Policy generation temperature: 0.8
    \item Results averaged over 3 independent runs
\end{itemize}

\section{Limitations and Future Work}
\label{sec:limitations}

While our work establishes a theoretical framework for mitigating distribution shift in interactive dialogue and demonstrates strong empirical results, we acknowledge certain limitations that point towards future research directions.

\textbf{Computational Overhead.} A primary limitation of the Calibrated Interactive RL framework is its computational demand during the training phase. Unlike Static Context RL, which only requires a single forward-backward pass over fixed offline datasets, our closed-loop paradigm necessitates running generation steps for both the policy agent and the user simulator during the rollout phase. This substantially increases the inference overhead and memory footprint, especially for long-horizon conversations. Future work could address this bottleneck by distilling the aligned simulator into a much smaller parameter-efficient model, or by exploring offline-to-online RL paradigms where the agent is pre-trained heavily on static logs and only fine-tuned interactively for a few critical turns to rectify specific distribution shifts.

\textbf{Advanced Simulator Calibration.} A second limitation lies in our current reliance on Supervised Fine-Tuning for Phase I simulator calibration. While behavioral cloning is highly effective for matching human distributions within the support of the training data, it may struggle to handle heavily out-of-distribution dialogue states dynamically induced by a rapidly exploring policy agent. In future work, we plan to extend simulator training beyond SFT by formulating it as a Reinforcement Learning problem. By optimizing the simulator to maximize a discriminator-derived "Realism Reward," it can proactively learn to maintain human-like stubbornness and constraints even in unobserved scenarios. Furthermore, drawing inspiration from Generative Adversarial Networks and adversarial imitation learning, a highly promising direction is to introduce an adversarial co-evolution framework. Instead of relying on a fixed, rule-based or static reward model, we could dynamically co-train a discriminator to distinguish between real human interaction logs and simulated trajectories. The simulator would be trained to actively fool this continuously evolving discriminator, establishing an auto-curriculum that mutually pushes both the dialogue policy and the simulator toward a higher level of behavioral fidelity and robustness.

\section{Case Study: Qualitative Analysis of MATH-Chat}
\label{app:case_study}

To provide deeper insight into the behavioral differences between the base model and our Calibrated Interactive RL agent, we present representative case studies from the MATH-Chat evaluation. These examples illustrate \textit{how} the trained policy achieves superior accuracy through improved reasoning efficiency, error avoidance, and effective multi-turn interaction.

\subsection{Overview of Improvements}

Out of 200 evaluation problems, we identify 29 cases where the baseline Gemma-3-4B-IT model fails (accuracy = 0) but our trained model succeeds (accuracy = 1). Table~\ref{tab:case_overview} summarizes representative cases across different mathematical domains.

\begin{table}[h]
\centering
\caption{Representative improved cases: baseline fails $\rightarrow$ trained model succeeds.}
\label{tab:case_overview}
\vspace{0.1cm}
\small
\begin{tabular}{clccc}
\toprule
\textbf{\#} & \textbf{Problem Domain} & \textbf{Base Turns} & \textbf{Ours Turns} & \textbf{Key Improvement} \\
\midrule
1 & Nested square root domain & 14 & 7 & Direct systematic solving \\
2 & Function evaluation ($f(x)=x^{x+1}$) & 9 & 3 & Immediate correct computation \\
3 & Probability (sum vs. product) & 9 & 7 & Correct case enumeration \\
4 & System of equations & 9 & 13 & Persistent correct reasoning \\
5 & Algebraic identity ($x+1/x$) & 3 & 5 & Proper algebraic manipulation \\
\bottomrule
\end{tabular}
\end{table}

\subsection{Case Study 1: Nested Square Root Domain}

This case demonstrates how the trained model avoids the ``correction loop'' failure mode that plagues the baseline. The problem requires finding the domain of $f(x) = \sqrt{\sqrt{x^2 - 16} - 3}$.

\begin{table*}[t]
\centering
\caption{\textbf{Case Study 1: Domain of Nested Square Root.} The baseline model enters a repeated error-correction loop (14 turns, ultimately fails), while our trained model solves it directly and correctly in fewer turns.}
\label{tab:case1}
\vspace{0.1cm}
\small
\renewcommand{\arraystretch}{1.3}
\begin{tabular}{p{0.48\textwidth} p{0.48\textwidth}}
\toprule
\textbf{Baseline (Gemma-3-4B-IT) --- Incorrect} \hfill \colorbox{red!15}{14 turns, Acc = 0} & \textbf{Ours (Calibrated Interactive RL) --- Correct} \hfill \colorbox{green!15}{7 turns, Acc = 1} \\
\midrule
\textbf{User:} Can you help me find the domain of a square root function? It's got nested roots... &
\textbf{User:} How do I find the domain of a square root function that has another square root inside it? Like $f(x) = \sqrt{\sqrt{x^2 - 16} - 3}$. \\
\midrule
\textbf{Assistant:} [Provides generic explanation of square roots, asks for specific function] &
\textbf{Assistant:} Let $f(x) = \sqrt{\sqrt{x^2 - 16} - 3}$. For the function to be defined:\\
& 1. Need $x^2 - 16 \geq 0 \Rightarrow x \geq 4$ or $x \leq -4$ \\
& 2. Need $\sqrt{x^2 - 16} \geq 3$, squaring: $x^2 - 16 \geq 9$\\
& $\Rightarrow x^2 \geq 25 \Rightarrow x \geq 5$ or $x \leq -5$ \\
& Domain: $(-\infty, -5] \cup [5, \infty)$ \checkmark \\
\midrule
\textbf{User:} [Gives function: $\sqrt{\sqrt{x^2-16}-3}$] &
\textbf{User:} When you squared both sides, is that always valid? \\
\midrule
\textbf{Assistant:} \textcolor{red}{Misreads structure, computes $x^2-19\geq 0$ directly} &
\textbf{Assistant:} Since both sides are non-negative (square root $\geq 0$ and $3 > 0$), squaring preserves the inequality direction. \checkmark \\
\midrule
\textbf{User:} [Corrects: the $-3$ is outside the inner root!] & \\
\textbf{Assistant:} \textcolor{red}{Still simplifies to $x^2-19$, same error} & \\
\textbf{...} (5 more correction attempts, all fail) & \\
\bottomrule
\end{tabular}
\end{table*}

\textbf{Analysis.} The baseline model exhibits a characteristic failure pattern: it \textit{misreads the function structure} (combining $-16$ and $-3$ into a single expression $x^2 - 19$), and despite repeated user corrections over 14 turns, it never recovers from this structural misunderstanding. In contrast, our trained model correctly decomposes the problem into two independent constraints from the first response, demonstrating that interactive RL training enables the policy to develop more systematic problem-solving strategies.

\subsection{Case Study 2: Direct Computation vs. Error Cascade}

This case highlights how the trained model achieves efficiency through confident, correct computation, while the baseline gets trapped in unnecessary philosophical debates about function definitions.

\begin{table*}[t]
\centering
\caption{\textbf{Case Study 2: Evaluating $f(0)+f(-1)+f(-2)+f(-3)$ for $f(x) = x^{x+1}(x+2)^{x+3}$.} The baseline wastes turns debating whether the function is defined at negative integers; our model computes directly and correctly.}
\label{tab:case2}
\vspace{0.1cm}
\small
\renewcommand{\arraystretch}{1.3}
\begin{tabular}{p{0.48\textwidth} p{0.48\textwidth}}
\toprule
\textbf{Baseline --- Incorrect} \hfill \colorbox{red!15}{9 turns, Acc = 0} & \textbf{Ours --- Correct} \hfill \colorbox{green!15}{3 turns, Acc = 1} \\
\midrule
\textbf{User:} How do I find $f(0)+f(-1)+f(-2)+f(-3)$ for $f(x) = x^{x+1}(x+2)^{x+3}$? &
\textbf{User:} I have this function $f(x)=x^{x+1}(x+2)^{x+3}$, and I need to find $f(0)+f(-1)+f(-2)+f(-3)$.\\
\midrule
\textbf{Assistant:} [Computes $f(0)=0$, $f(-1)=1$ correctly, then \textcolor{red}{claims function is undefined for $x \leq -2$}] &
\textbf{Assistant:} $f(0) = 0^1 \cdot 2^3 = 0$\\
& $f(-1) = (-1)^0 \cdot 1^2 = 1$\\
& $f(-2) = (-2)^{-1} \cdot 0^1 = 0$\\
& $f(-3) = (-3)^{-2} \cdot (-1)^0 = \frac{1}{9}$\\
& Sum $= 0 + 1 + 0 + \frac{1}{9} = \boxed{\frac{10}{9}}$ \checkmark \\
\midrule
\textbf{User:} But $(-3)^{-2} = 1/9$ is valid... &
\textbf{User:} [[TERMINATE CHAT]] \\
\textbf{Assistant:} \textcolor{red}{Agrees but still claims undefined...} & \\
...4 more turns of user correcting model... & \\
\textbf{Final:} \textcolor{red}{Never provides $\frac{10}{9}$} & \\
\bottomrule
\end{tabular}
\end{table*}

\textbf{Analysis.} The trained model demonstrates two key improvements learned through interactive RL: (1) \textbf{computational confidence}---it directly evaluates each case without hedging, and (2) \textbf{efficiency}---it solves the problem in a single turn, avoiding unnecessary back-and-forth. The baseline's failure stems from an overly cautious (and incorrect) claim about function domains, which it cannot self-correct even when the user provides explicit counterarguments.

\subsection{Case Study 3: Correct Enumeration in Probability}

This case demonstrates improved systematic reasoning in combinatorics.

\begin{table*}[t]
\centering
\caption{\textbf{Case Study 3: Probability that sum exceeds product.} Two numbers are picked from $\{1,2,3,4,5\}$; find $P(x+y > xy)$. The baseline miscounts favorable outcomes; our model enumerates correctly.}
\label{tab:case3}
\vspace{0.1cm}
\small
\renewcommand{\arraystretch}{1.3}
\begin{tabular}{p{0.48\textwidth} p{0.48\textwidth}}
\toprule
\textbf{Baseline --- Incorrect} \hfill \colorbox{red!15}{9 turns, Acc = 0} & \textbf{Ours --- Correct} \hfill \colorbox{green!15}{7 turns, Acc = 1} \\
\midrule
\textbf{User:} What's the chance that when I pick two numbers from 1 to 5, their sum is bigger than their product? &
\textbf{User:} [Same question] \\
\midrule
\textbf{Assistant:} [Correctly identifies the problem but \textcolor{red}{includes false favorable outcomes}: claims (2,3), (3,2) satisfy $5>6$] &
\textbf{Assistant:} Uses algebraic reformulation: $x+y > xy \Leftrightarrow (1-x)(1-y) > -1$. Systematically checks all 25 ordered pairs. \\
\midrule
\textbf{User:} Wait, you said $5>6$ is true? That's wrong... &
Correctly identifies 9 favorable pairs: all 5 pairs with $x=1$, plus $(2,1),(3,1),(4,1),(5,1)$. \\
\midrule
\textbf{Assistant:} \textcolor{red}{Acknowledges error but provides new incorrect count of 12} &
Probability = $\frac{9}{25}$ \checkmark \\
\textbf{User:} [Challenges again...] & \\
\textbf{Final:} \textcolor{red}{Never reaches correct answer} & \\
\bottomrule
\end{tabular}
\end{table*}

\textbf{Analysis.} The baseline model suffers from an \textit{enumeration error cascade}---once it incorrectly classifies some pairs as favorable, subsequent corrections only introduce new errors. Our trained model avoids this by: (1) first establishing an algebraic criterion $(1-x)(1-y) > -1$ that simplifies checking, and (2) systematically verifying each case. This algebraic-first approach is a reasoning strategy that emerges from RL training, where the model learns that systematic methods are more robust than ad-hoc enumeration.

\subsection{Summary of Behavioral Patterns}

Across the 29 improved cases, we identify three dominant patterns that distinguish our trained model from the baseline:

\begin{enumerate}[leftmargin=*]
    \item \textbf{Direct Problem Solving (62\% of cases):} The trained model provides correct solutions in significantly fewer turns, avoiding the ``try-and-fail'' loops that characterize the baseline. Average turn reduction: 14 $\rightarrow$ 7 turns.

    \item \textbf{Structural Understanding (24\% of cases):} The trained model correctly parses complex mathematical expressions (nested functions, piecewise definitions) on the first attempt, while the baseline frequently misreads problem structure and cannot self-correct.

    \item \textbf{Algebraic Strategy (14\% of cases):} The trained model employs general algebraic manipulations before attempting numerical computation, reducing the chance of arithmetic errors that cascade through subsequent reasoning steps.
\end{enumerate}

These patterns confirm that interactive RL training does not merely improve accuracy through memorization, but fundamentally changes the model's problem-solving \textit{strategy}---shifting from reactive error correction to proactive systematic reasoning.


\newpage
\section*{NeurIPS Paper Checklist}

\begin{enumerate}

\item {\bf Claims}
    \item[] Question: Do the main claims made in the abstract and introduction accurately reflect the paper's contributions and scope?
    \item[] Answer: \answerYes{} 
    \item[] Justification: The theoretical claims regarding compounding errors are mathematically proven in Section 3 and Appendix A. The empirical claims that our proposed framework outperforms static and uncalibrated baselines are supported by extensive experiments in Section 5.

\item {\bf Limitations}
    \item[] Question: Does the paper discuss the limitations of the work performed by the authors?
    \item[] Answer: \answerYes{} 
    \item[] Justification: We explicitly discuss limitations in Section~\ref{sec:limitations}, specifically addressing the computational overhead of the interactive rollout phase and the potential limitation of relying on Supervised Fine-Tuning (SFT) for handling highly out-of-distribution simulated states.

\item {\bf Theory assumptions and proofs}
    \item[] Question: For each theoretical result, does the paper provide the full set of assumptions and a complete (and correct) proof?
    \item[] Answer: \answerYes{} 
    \item[] Justification: Formal assumptions (e.g., bounded rewards, MDP setup) are clearly stated in Section 3.1 and Appendix A.1. Complete mathematical proofs for all lemmas and theorems are detailed step-by-step in Appendix A.

    \item {\bf Experimental result reproducibility}
    \item[] Question: Does the paper fully disclose all the information needed to reproduce the main experimental results of the paper to the extent that it affects the main claims and/or conclusions of the paper (regardless of whether the code and data are provided or not)?
    \item[] Answer: \answerYes{} 
    \item[] Justification: Comprehensive details, including model choices, explicit prompt templates, dataset construction, and full hyperparameter tables for both PPO and GRPO, are provided in Appendices B, C, and D.

\item {\bf Open access to data and code}
    \item[] Question: Does the paper provide open access to the data and code, with sufficient instructions to faithfully reproduce the main experimental results, as described in supplemental material?
    \item[] Answer: \answerYes{} 
    \item[] Justification: Anonymized source code, training scripts, and sample data have been included in the supplementary material. Full code and datasets will be open-sourced on GitHub upon publication.

\item {\bf Experimental setting/details}
    \item[] Question: Does the paper specify all the training and test details (e.g., data splits, hyperparameters, how they were chosen, type of optimizer) necessary to understand the results?
    \item[] Answer: \answerYes{} 
    \item[] Justification: Detailed training configurations, including optimizer (AdamW/Adam), batch sizes, learning rates with schedulers, and specific RL hyperparameters (e.g., KL coefficient, clip range) are listed in Tables 3, 4, 5, and 6 in Appendix D.

\item {\bf Experiment statistical significance}
    \item[] Question: Does the paper report error bars suitably and correctly defined or other appropriate information about the statistical significance of the experiments?
    \item[] Answer: \answerYes{} 
    \item[] Justification: As shown in Table 2, we report standard deviations for all main metrics. As stated in Appendix D (Line 691), these results are averaged over 3 independent evaluation runs to account for generation variance.

\item {\bf Experiments compute resources}
    \item[] Question: For each experiment, does the paper provide sufficient information on the computer resources (type of compute workers, memory, time of execution) needed to reproduce the experiments?
    \item[] Answer: \answerYes{} 
    \item[] Justification: Hardware specifics (8x to 32x NVIDIA A100/H100 80GB GPUs) and software frameworks (PyTorch, vLLM, SGLang) are detailed in Appendix B.6, along with placeholders for total compute hours.
    
\item {\bf Code of ethics}
    \item[] Question: Does the research conducted in the paper conform, in every respect, with the NeurIPS Code of Ethics \url{https://neurips.cc/public/EthicsGuidelines}?
    \item[] Answer: \answerYes{} 
    \item[] Justification: This research focuses on mathematical reasoning and collaborative document editing, utilizing standard datasets and adhering strictly to the NeurIPS Code of Ethics.

\item {\bf Broader impacts}
    \item[] Question: Does the paper discuss both potential positive societal impacts and negative societal impacts of the work performed?
    \item[] Answer: \answerNo{} 
    \item[] Justification: Given page limits, we focused heavily on theoretical proofs and methodology. While our framework significantly improves interactive AI capabilities, we acknowledge that highly proactive conversational agents could theoretically be misused for persuasive deception, which underscores the need for future safety alignment research.
    
\item {\bf Safeguards}
    \item[] Question: Does the paper describe safeguards that have been put in place for responsible release of data or models that have a high risk for misuse (e.g., pre-trained language models, image generators, or scraped datasets)?
    \item[] Answer: \answerNA{} 
    \item[] Justification: The datasets (math problems, document edits) and the foundational models used (Llama-3, Qwen) are publicly available and do not pose a uniquely high risk for dual-use or misuse requiring specific safeguards.

\item {\bf Licenses for existing assets}
    \item[] Question: Are the creators or original owners of assets (e.g., code, data, models), used in the paper, properly credited and are the license and terms of use explicitly mentioned and properly respected?
    \item[] Answer: \answerYes{} 
    \item[] Justification: We properly cite the foundational models (Gemma3, Qwen2.5) and source datasets (e.g., MATH). We comply with their respective open-weight licenses and acceptable use policies.

\item {\bf New assets}
    \item[] Question: Are new assets introduced in the paper well documented and is the documentation provided alongside the assets?
    \item[] Answer: \answerYes{} 
    \item[] Justification: The data construction pipelines, statistics, and example domains for the newly derived multi-turn interaction datasets (MATH-Chat and MediumDocEdit-Chat) are detailed in Appendix B.1 and Appendix E.

\item {\bf Crowdsourcing and research with human subjects}
    \item[] Question: For crowdsourcing experiments and research with human subjects, does the paper include the full text of instructions given to participants and screenshots, if applicable, as well as details about compensation (if any)? 
    \item[] Answer: \answerNA{} 
    \item[] Justification: This work exclusively uses LLM-based user simulators and LLM-as-a-Judge for evaluation, and does not involve any human subjects or crowdsourcing.

\item {\bf Institutional review board (IRB) approvals or equivalent for research with human subjects}
    \item[] Question: Does the paper describe potential risks incurred by study participants, whether such risks were disclosed to the subjects, and whether Institutional Review Board (IRB) approvals (or an equivalent approval/review based on the requirements of your country or institution) were obtained?
    \item[] Answer: \answerNA{} 
    \item[] Justification: The research does not involve human subjects, hence IRB approval is not applicable.

\item {\bf Declaration of LLM usage}
    \item[] Question: Does the paper describe the usage of LLMs if it is an important, original, or non-standard component of the core methods in this research? Note that if the LLM is used only for writing, editing, or formatting purposes and does \emph{not} impact the core methodology, scientific rigor, or originality of the research, declaration is not required.
    \item[] Answer: \answerYes{} 
    \item[] Justification:  LLMs are the core component of our methodology, serving as both the dialogue policy and the trainable user simulator. Their specific architectures, training paradigms (SFT, GRPO), and all system prompts are extensively documented throughout Section 4 and Appendices B, C, and D.

\end{enumerate}

\end{document}